\documentclass[10pt, conference, letterpaper]{IEEEtran}
\IEEEoverridecommandlockouts

\usepackage{cite}
\usepackage{amsthm}
\usepackage{amsmath,amssymb,amsfonts}
\usepackage{algorithm}
\usepackage{algpseudocode}
\usepackage{graphicx}
\usepackage{textcomp}
\usepackage{xcolor}
\usepackage{caption}
\usepackage{subcaption}
\usepackage{xspace}
\usepackage{booktabs}
\usepackage{multirow}
\usepackage{comment}
\usepackage{url}

\usepackage{float}
\newfloat{algorithm}{t}{lop}

\newtheorem{theorem}{Theorem}[section]

\newtheorem{lemma}[theorem]{Lemma}
\newtheorem{definition}{Definition}[section]

\newcommand{\dadmm}{\texttt{D-ADMM}\xspace}
\newcommand{\avgestimate}{\bar{z}}

\newcommand{\sysname}{\texttt{SecureD-FL}\xspace}

\begin{document}

\title{Privacy-preserving Decentralized Aggregation for Federated Learning}

\author{
\IEEEauthorblockN{Beomyeol Jeon\textsuperscript{*}\textsuperscript{1}, S. M. Ferdous\textsuperscript{*}\textsuperscript{2}, 
Muntasir Raihan Rahman\textsuperscript{$\dagger$}\textsuperscript{3}, 
and Anwar Walid\textsuperscript{4}}
\IEEEauthorblockA{\textit{
\textsuperscript{1}University of Illinois at Urbana-Champaign,
\textsuperscript{2}Purdue University,
\textsuperscript{3}Microsoft,
\textsuperscript{4}Nokia Bell Labs}
}
\textsuperscript{1}bj2@illinois.edu,
\textsuperscript{2}sferdou@purdue.edu,
\textsuperscript{3}Muntasir.Rahman@microsoft.com,
\textsuperscript{4}anwar.walid@nokia-bell-labs.com
}

\maketitle
\begingroup\renewcommand\thefootnote{*}
\footnotetext{Both authors contributed equally to this work.}
\endgroup
\begingroup\renewcommand\thefootnote{$\dagger$}
\footnotetext{This work was done when the author was at Nokia Bell Labs.}
\endgroup

\begin{abstract}
Federated learning is a promising framework for learning over decentralized data spanning multiple regions. 
This approach avoids expensive central training data aggregation cost and can improve privacy because distributed sites do not have to reveal privacy-sensitive data.
In this paper, we develop a privacy-preserving decentralized aggregation protocol for federated learning.
We formulate the distributed aggregation protocol with the Alternating Direction Method of Multiplier (ADMM) and examine its privacy weakness.
Unlike prior work that use Differential Privacy or homomorphic encryption for privacy,
we develop a protocol that controls communication among participants in each round of aggregation to minimize privacy leakage.
We establish its privacy guarantee against an honest-but-curious adversary.
We also propose an efficient algorithm to construct such a communication pattern, inspired by combinatorial block design theory.
Our secure aggregation protocol based on this novel group communication pattern design leads to an efficient algorithm for federated training with privacy guarantees.
We evaluate our federated training algorithm on image classification and next-word prediction applications over benchmark datasets with 9 and 15 distributed sites.
Evaluation results show that our algorithm performs comparably to the standard centralized federated learning method while preserving privacy; the degradation in test accuracy is only up to 0.73\%.
\end{abstract}


\section{Introduction}
In various IoT and networking applications, data is collected and stored at the source, and there is an interest in utilizing distributed data to train machine learning (ML) models.
However, due to transport costs, delays, and privacy concerns, the data may not be moved to a central location for processing and learning.
For example, several factory sites, employing same robots, are interested in training an ML model for a robot maintenance scheduling plan that can benefit each competing factory site without central aggregation of privacy sensitive data.
To improve the scheduling plan quality, the factories may want to utilize all data collected from robots at all sites. Still, they may not want to share privacy-sensitive raw data or move it to a central location.
Similar use-cases arise in hospital settings, where different hospitals want to improve diagnosis accuracy by utilizing the privacy sensitive patient data of all hospitals without central aggregation.

\emph{Federated learning} (FL)~\cite{mcmahan2017communication, konen2016federated} has been proposed as a promising framework for learning over decentralized data spanning multiple regions that would allow such decentralized private sites to learn a global model that trains over all data across sites without central aggregation. 
Instead of transferring training data, each site trains identical ML models leveraging their local data.
A central server aggregates these locally trained models to generate a global model and distributes it over all the sites.
In FL, private local data stays local; only the trained model parameters move. Thus it provides the privacy of local data and also avoids expensive data transfer.

Although the standard FL only exchanges model parameters, an adversary can still infer privacy-sensitive information from the leaked model parameters~\cite{shokri2017membership,nasr2019comprehensive,melis2019exploiting}.
Even with only the trained model, an attacker can infer whether a particular data item belongs to the dataset used for the ML model training (membership inference attack~\cite{shokri2017membership}).

Techniques to address these privacy concerns include Secure Multi-party Computation~\cite{bonawitz2017practical, chen2019secure, agrawal2019quotient}, Differential Privacy~\cite{shokri2015privacy, abadi2016dldifferentialprivacy}, and combinations of both~\cite{truex2018hybrid, hybridalpha, pysyft}.
However such approaches incur large computation overheads, require a trusted third party for the secret key generation, or sacrifice the quality of trained models due to the introduced noise. 
Importantly, these approaches require a central aggregation server. If the server is compromised, the leaked ML models cause privacy risks. 
Besides, the central aggregator can be a single point of failure.
Many devices, simultaneously uploading model updates to the server, cause TCP incast~\cite{chen2009incast}, which can slow down the training process.



To avoid challenges with a central server, we choose Alternating Direction Method of Multiplier (ADMM) for decentralized aggregation.
Several approaches have been proposed to augment ADMM with privacy guarantees by using Differential Privacy \cite{Zhang2017dynamic, Huang2020dpadmm, zhang2018improving} and homomorphic encryption \cite{zhang2018admm}.
In this work, we propose a different approach: \emph{control a communication pattern among participants}. The main idea is that we develop a protocol to decide which group of parties should communicate in each round of aggregation to minimize privacy leakage.
We establish its theoretical privacy guarantee against an honest-but-curious adversary. 
Our novel communication pattern generation is inspired by combinatorial block design theory~\cite{stinson2007combinatorial}.
To the best of our knowledge, we are the first to apply combinatorial design theory to develop group communication patterns for secure decentralized aggregation.
Our secure aggregation protocol forms the basis of a federated learning algorithm that is privacy preserving and efficient.
Our contributions in this paper are as follows.


\begin{itemize}
\item We show that classic ADMM-based aggregation algorithm has a privacy risk against an \textit{honest-but-curious} adversary.

\item We propose a scheme that generates disjoint groups and allows communication only within a group in each ADMM iteration. This preserves privacy and linear convergence.
\item We introduce \emph{gap}, an approximate measure of privacy, which provides a trade-off between privacy and accuracy. The gap represents the number of iterations between two devices being in the same group. If the gap is at least $k$, ADMM can run securely for at most $2k-1$ iterations without revealing private local model parameters.

\item We show that our desired group formation (with the gap constraints) is equivalent to a class of \textit{resolvable balanced incomplete block design} problems in combinatorial design theory~\cite{stinson2007combinatorial}. To the best of our knowledge, we are the first to explore this connection. 
We also propose an efficient randomized algorithm for the group formation.

\item We design a FL algorithm with our secure aggregation and evaluate it on image classification and language models over the benchmark datasets.
Our decentralized secure FL shows training performance comparable to the 
centralized FL~\cite{mcmahan2017communication} with two-round aggregation.
We also analyze the trade-off between the number of communication rounds in aggregation and the estimation quality.
We evaluate the performance of our proposed randomized group construction algorithm.

\end{itemize}

\section{Distributed Averaging by ADMM}
We employ an optimization technique to address the privacy concern in aggregating the local model parameters. 
In this section, we will first formulate the aggregation problem as a standard distributed optimization problem. Then we will explore a privacy issue in the optimization formulation. First, let us define the preliminaries. The number of users (or data sites) is denoted by $n$. 
In FL, each user $k$ has a local model parameter vector $w_k$~\footnote{For simplicity, a vector is used but this generalizes to a tensor.}.
The global model parameter is computed by averaging local parameters. 
Let $x$ be a variable of equivalent dimension of $w_k$. Then we can formulate an optimization problem as follows. 
    \begin{align}
    \label{lab:prob1}
    & {\text{minimize}} & \sum_{k=1}^n ||(x-w_k)||_2^2
    \end{align}

The optimal solution of this unconstrained least square problem is the average, i.e., $x^* = \frac{1}{n}{\sum_{k=1}^n w_k}$.
To solve this problem in a distributed way, we follow the standard distributed consensus technique described in \cite{Boyd+2011}. We introduce local variables $x_k$ for each user $k \in [n]$ and a consensus variable $z$. An equivalent optimization problem of Equation~\ref{lab:prob1} is as follows.

    \begin{align}
    \label{lab:prob2}
    & {\text{minimize}} & \sum_{k=1}^n ||(x_k-w_k)||_2^2\\
    & \text{subject to} & x_k = z, \; k=1, \ldots, n. \nonumber
    \end{align}

We compute the augmented Lagrangian $L_\rho(x_k,\lambda_k, z)$ as,
\begin{equation}
\sum_{k=1}^n (||x_k-w_k||_2^2  + \lambda_k^T(x_k-z) + \frac{\rho}{2} || x_k-z||^2) 
\end{equation}

Here, $\lambda_k$ is a dual variable defined for each user, and $\rho \in \mathbb{R}$ is a penalty parameter.
    
Following the standard ADMM technique as in \cite{Boyd+2011}, the $x_k$-minimization, the $z$-minimization, and $\lambda_k$ updates can be written as follows.
\begin{align}
    & x_k^{i} = \frac{1}{2+\rho} (2w_k-\lambda_k^{i-1} + \rho z^{i-1}) \label{eqn:x_min}\\
    & z^{i} = \frac{1}{n} \sum_{k=1}^n (x_k^{i} + \frac{1}{\rho} \lambda_k^{i-1} ) \\
    & \lambda_k^{i} = \lambda_k^{i-1} + \rho \cdot (x_k^{i}-z^{i})
\end{align}
Here $i$ is an iteration counter.

\subsection{The ADMM-based Aggregation Algorithm}

\begin{algorithm}[H]
\caption{ADMM-based Distributed Averaging}
\label{alg:admm}
\begin{algorithmic}
\State Initialize $\lambda_k^0$ for each user $k$
\State $z^0 = 0$

\For{iteration $i=1,2\dots$}\Comment{Until stopping criteria met}
    \ForAll{user $k$ {\bf in parallel}}
        \State $x_k^{i} = \frac{1}{2+\rho} (2w_k-\lambda_k^{i-1} + \rho z^{i-1})$ \Comment{$x_k$ minimization}
        \State Send $y_k^i = x_k^{i} + \frac{1}{\rho} \lambda_k^{i-1}$ to all
        \State $z^{i} = \frac{1}{n} \sum_{k=1}^n y_k^{i} $ \Comment{$z$ minimization}
        \State $\lambda_k^{i} = \lambda_k^{i-1} + \rho \cdot (x_k^{i}-z^{i})$ \Comment{$\lambda_k$ update}
    \EndFor
\EndFor


\end{algorithmic}
\end{algorithm}

Algorithm~\ref{alg:admm} shows the ADMM-based distributed averaging. Each user $k$ initializes its dual variable $\lambda_k^0$ randomly and sets the consensus variable $z^0$ to zero. In iteration $i$, user $k$ calculates $x_k^{i}$ by minimizing the augmented Lagrangian using $z^{i-1}$ and $\lambda_k^{i-1}$ (Equation~\ref{eqn:x_min}). Then users communicate each other to update the consensus variable $z^i$. 
Finally, the user $k$ updates its dual variable $\lambda_k^{i}$ with the updated $x_k^{i}$ and $z^i$. 
The iterations continue until the stopping criteria are met. 


\subsection{Proof of Convergence} 
Lemma~\ref{lemma:dual_zero} shows that, after the first iteration of Algorithm~\ref{alg:admm}, the summation of the all the dual variables becomes a zero vector. 
\begin{lemma}
\label{lemma:dual_zero}
$\forall{i \geq 1}$, $\sum_{k=1}^n \lambda_{k}^{i} = 0$
\end{lemma}

\begin{proof}
{
\begin{align*}
    \sum_{k=1}^n \lambda_{k}^{i} & =  \sum_{k=1}^n (\lambda_{k}^{i-1} + \rho \cdot (x_k^{i} - z^{i})) \\
                                 & = \sum_{k=1}^n (\lambda_{k}^{i-1} + \rho \cdot (x_k^{i} - \frac{1}{n} (\sum_{j=1}^n x_j^{i} + \frac{1}{\rho} \sum_{j=1}^n \lambda_{j}^{i-1})) \\
                                 & = 0
\end{align*}
}
\end{proof}
Now we will show that the consensus variable $z$ converges linearly to the true average. 
\begin{lemma} \label{lemma:convergence}
The consensus variable converges to the average linearly at a rate of $\frac{\rho}{\rho+2}$.
\end{lemma}
\begin{proof}
    We know the optimum value of the consensus variable $z^* = \frac{1}{n}\sum_{k=1}^n w_k$. Let us define a residual at $i$-th iteration, $\Delta z^{i}=z^{i} - z^*$.
    
    \begin{align*}
        \Delta z^{i} &=z^{i} - z^*\\
                       &= \frac{1}{n} (\sum_{k=1} ^n x_k^{i} + \frac{1}{\rho} \sum_{k=1}^n \lambda_k^{i-1}) - \frac{1}{n}\sum_{k=1}^n w_k\\
                       & = \frac{1}{n} \sum_{k=1}^n x_k^{i} -\frac{1}{n}\sum_{k=1}^n w_k\\
                      & = \frac{1}{n} \sum_{k=1}^n (\frac{1}{2+\rho} (2w_k -\lambda_k^{i-1} + \rho \cdot z^{i-1})) - \frac{1}{n}\sum_{k=1}^n w_k\\
                       & = \frac{\rho}{\rho+2} (z^{i-1} - \frac{1}{n}\sum_{k=1}^n w_k)\\
                       & = \frac{\rho}{\rho+2} (z^{i-1} - z^*) \\
          & \implies ||\Delta z^{i}|| \leq  \frac{\rho}{\rho+2} ||\Delta z^{i-1}|| 
    \end{align*}
\end{proof}

The third line follows Lemma~\ref{lemma:dual_zero}. $||\Delta z^{i}||$ represents $z^{i}$'s distance to $z^*$. 
Since $\frac{\rho}{\rho+2} < 1$, this achieves the linear convergence. 

\section{Privacy Preserving Averaging} \label{sec:secure_aggregation}

We see for our problem, ADMM is easily implementable in a distributed environment and has good convergence. In Algorithm~\ref{alg:admm}, the privacy preservation may be assumed since participants do not directly share their private data (i.e., $w_k$).
In \cite{WEERADDANA20179502}, the authors claim that ADMM has some privacy guarantee. But in this section, we will show that Algorithm~\ref{alg:admm} is insufficient to preserve privacy in our threat model: a \textit{semi-trust} or \textit{honest-but-curious} threat model.


\begin{definition}[Honest-but-curious Threat Model\footnote{The formal treatment of the model is provided in Chapter 7 of \cite{Oded2009}.}]
In the honest-but-curious threat model, the participants are assumed to follow the protocol (honest) 
but simultaneously accumulate all the information they have seen, e.g., the messages sent to them (curious). 
\end{definition}


In the honest-but-curious model, a secure protocol prevents the participants from inferring the privacy-sensitive data of others even if they gather all communicated messages.


\subsection{Privacy Analysis of the Classic ADMM-based Algorithm}
We show that Algorithm \ref{alg:admm} is not secure against \textit{honest-but-curious} participants.

\begin{lemma}
\label{lemma:privacy_1}
In all-to-all communication, any honest-but-curious participant following Algorithm \ref{alg:admm} can retrieve the privacy-sensitive data (i.e., $w$)
when it collects messages for at least two consecutive iterations.
\end{lemma}

\begin{proof}

Assume that a user $k$ is a \textit{honest-but-curious} participant and wants to know the private data of a user $k^\prime$, $w_{k^\prime}$. 
For simplicity, suppose that 
the $\rho$ value are shared among all participants.
Note that the user $k$ knows $z^1$ and $\rho$ but it does not know $\lambda_{k^\prime}^0$.
Throughout the proof, boldface terms are constant terms with respect to the user $k$: i.e., the user $k$ knows the values of them.

In the first iteration, the user $k$ receives a message from the user $k^\prime$, $\mathbf{y_{k^\prime}^1}$.



\begin{equation}
\label{eqn:admm_privacy1}
    \mathbf{y_{k^\prime}^1} = x_{k^\prime}^1+\frac{1}{\boldsymbol{\rho}}\lambda_{k^\prime}^0
\end{equation}

The user $k$ knows the consensus value $\mathbf{z^1}$, So it can calculate $\lambda_{k^\prime}^1$ as follows.

\begin{align}
\label{eqn:admm_privacy2}
    \lambda_{k^\prime}^1 &= \lambda_{k^\prime}^{0} + {\boldsymbol{\rho}} (x_{k^\prime}^1-\mathbf{z^1}) \nonumber\\
                &= \lambda_{k^\prime}^{0} + {\boldsymbol{\rho}}(\mathbf{y_{k^\prime}^1}-\frac{1}{\boldsymbol{\rho}} \lambda_{k^\prime}^{0} - \mathbf{z_1^1})\nonumber\\
                &= \boldsymbol{\rho} (\mathbf{y_{k^\prime}^1}-\mathbf{z^1}) 
\end{align}

At this point, the user $k$ cannot retrieve $w_{k^\prime}$ as Equation~\ref{eqn:admm_privacy1}  have two unknowns.
In the second iteration, the user $k$ receives a following message from the user $k^\prime$.

\begin{equation}
    \mathbf{y_{k^\prime}^2} = x_{k^\prime}^2+\frac{1}{\boldsymbol{\rho}}\mathbf{\lambda_{k^\prime}^1} \label{eq1}
\end{equation}

Since $k$ knows how $x_{k^\prime}$ is updated, it may plug the value of $x_{k^\prime}$ into the following equation and obtain the private data $w_{k^\prime}$.

\begin{align}
    \mathbf{x_{k^\prime}^{2}} = \frac{1}{2+\boldsymbol{\rho}} (2w_{k^\prime}-\mathbf{\lambda_{k^\prime}^1} + \boldsymbol{\rho} \mathbf{z^1}) \\
    \implies w_{k^\prime} = ((2+\boldsymbol{\rho}) \mathbf{x_{k^\prime}^{2}} + \mathbf{\lambda_{k^\prime}^1} - \boldsymbol{\rho} \mathbf{z^1}) / 2
\end{align}



\end{proof}

\subsection{Enhancing Privacy through the Gap in Communication}
We see that the privacy is not guaranteed in all-to-all communication (Lemma~\ref{lemma:privacy_1}).
A natural approach would then be to limit communication among users and minimize information to infer private data. We adopt this approach and systematically analyze it.

In Algorithm~\ref{alg:admm}, 
users communicate messages $y_k^i = x_k^i + \frac{1}{\rho} \lambda_k^{i-1}$, which is required to calculate the consensus variable $z^i$, the average of the $y_k^i$s.
In all-to-all communication, $z^i$ can be computed in a single communication round. 

However, we observe that all-to-all communication is not required. An alternative communication pattern can be used. 
For instance, we can partition users into groups; users only communicate their $y_k^i$ within a group $g$ and compute the intermediate $z_g^i$, the average $y_k^i$s in the group. Next, they communicate the intermediate $z_g^i$ across groups to compute the mean of $z_g^i$s, which turns to be the final $z^i$. 

This communication scheme
does not affect the ADMM convergence. Interestingly, this group-based communication approach provides the desired \emph{privacy guarantee} if the groups follow a certain gap constraint.


\begin{definition}[\textbf{Gap}]
 Given any two users, a gap is the number of consecutive iterations, after which they communicate their $y$ messages.
\end{definition}

\begin{definition}[\textbf{Group}]
 A group of users is the participants who communicate their $y$ messages with each other in an iteration.
\end{definition}

We denote $t_g$ as the gap size, $g$ as the set of groups and $s$ as the number of participants in a group. Note that in Algorithm~\ref{alg:admm} $t_g=1$, $|g|=1$, and $s=n$.

A gap of $t_g$ means that no two participants can be in the same group for consecutive $t_g$ iterations of the ADMM aggregation. Let us now analyze the ADMM algorithm with the gap and group properties. Let $T_p$ be the number of iterations after which the privacy is revealed. 


\begin{theorem}
\label{lemma:max_iter}
       Assuming $s> \frac{t_g}{t_g-1}$, $T_p = 2t_g$
\end{theorem}

\begin{proof}
Let  $T$ and $i$ be the number of ADMM iterations and an iteration counter, respectively. Since we are free to choose $T$ (without loss of generality), we assume $T$ is a multiple of $t_g$. Also let $k$ be an \emph{honest-but-curious} participant who wants to learn the model parameter of $k^\prime$. From $k$'s point of view, the best case is when $k$ and $k^\prime$ are part of the same group for $T$ iterations, which provides the maximum information about $k^\prime$. At iteration $i$, $k$ has 3 equations concerning $k^\prime$ as follows.
\begin{align}
    & x_{k^\prime}^{i} = \frac{1}{2+\rho} (2w_{k^\prime}-\lambda_{k^\prime}^{i-1} + \rho z^{i-1}) \label{eq:gap_eq1} \\
    & y_{k^\prime}^{i} = x_{k^\prime}^{i} + \frac{1}{\rho} \lambda_{k^\prime}^{i-1} \label{eq:gap_eq2} \\
    & \lambda_{k^\prime}^{i} = \lambda_{k^\prime}^{i-1} + \rho \cdot (x_{k^\prime}^{i}-z^{i}) \label{eq:gap_eq3}
\end{align}

This gives a total $3T$ equations for $T$ iterations. 
Let us analyze the number of unknowns in the system of these $3T$ equations. In Equation~\ref{eq:gap_eq1}, $w_{k^\prime}$ is unknown across all the  $T$ iterations. Apart from that, we have two more  unknowns per iteration: $x_{k^\prime}^{i+1}$ and $\lambda_{k^\prime}^i$. 

Now let us turn to Equation~\ref{eq:gap_eq2}. Since the gap size is $t_g$, $k$ and $k^\prime$ could be in the same group in at most $\frac{T}{t_g}$ iterations. Let $P$ be the set of these iteration counters. 
For $j \in P$, $k$ knows 
$y_{k^\prime}^j$ messages. This leaves $(T-\frac{T}{t_g})$ new unknowns in Equation~\ref{eq:gap_eq2}. 

In Equation~\ref{eq:gap_eq3}, $\lambda_{k^\prime}^{T+1}$  is a new variables at the $T$-th iteration. Summing all these, the number of total unknowns across the $T$ iterations is $1+2T + (T-\frac{T}{t_g}) + 1= \frac{3Tt_g-T}{t_g}+2$.

Retrieving the private data of $k^\prime$ is now equivalent to find a unique solution of this system of $3T$ linear equations. In general the unique solution exists if the number of unknowns matches to the number of equations as follows. 

\begin{align}
\begin{split}
    \frac{3T_pt_g-T_p}{t_g}+2&=3T_p\\
    \implies T_p&=2t_g
\end{split}
\end{align}

We also have $T$ more equations concerning $y_{k^\prime}^i$, the intermediate group mean of the form $z_{g_j}=\sum_{u \in g^i} y_u^i$. Here $g^i$ is the group in which $k^\prime$ belongs at $i$-th iteration. 
When $k$ and $k^\prime$ are in the same group (this happens at most in $\frac{T}{t_g}$ iterations), there are no unknowns. 
In  other $(T-\frac{T}{t_g})$ iterations, the number of unknowns is equal to the size of the group that $k^\prime$ belongs to (i.e., $s$). So the total number of unknowns here is at least $ (T-\frac{T}{t_g})s$.
We find out the condition of $s$ with which this system of equations becomes over-determined.

\begin{align}
\centering
\begin{split}
    &(T-\frac{T}{t_g})s > T \\
    &\equiv s  > \frac{t_g}{t_g-1}
\end{split}
\end{align}

This holds if $t_g>1$. Assuming $s > \frac{t_g}{t_g-1}$, the intermediate equations can be discarded from the attacker's point of view.

\end{proof}






\section{The ``Group'' construction algorithm}

\begin{figure}
    \centering
    \hfill
    \begin{subfigure}[b]{0.4\linewidth}
    \includegraphics[width=\linewidth]{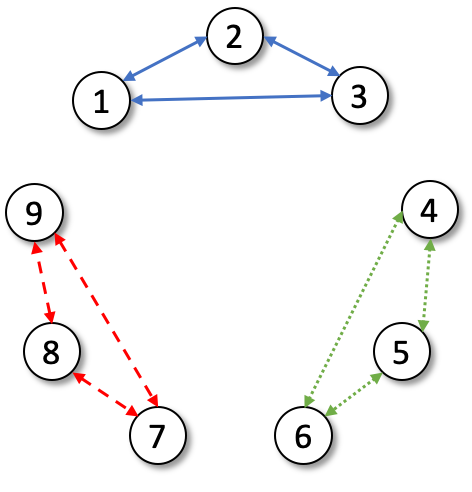}
    \end{subfigure}
    \hfill
    \begin{subfigure}[b]{0.4\linewidth}
    \includegraphics[width=\linewidth]{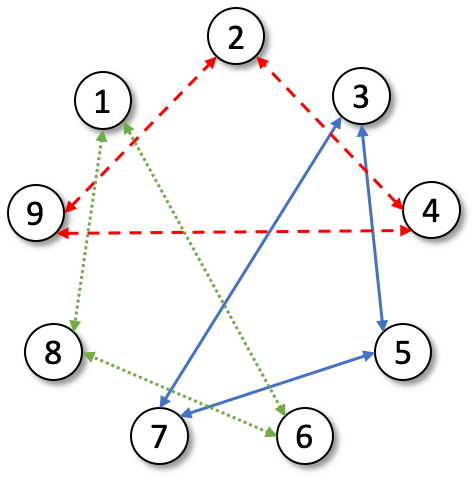}
    \end{subfigure}
    \hfill

    \hfill
    \begin{subfigure}[b]{0.4\linewidth}
    \includegraphics[width=\linewidth]{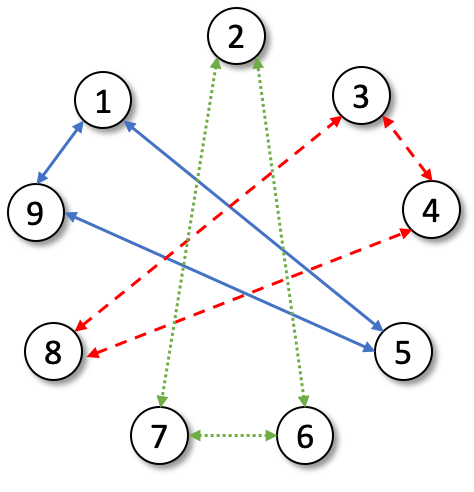}
    \end{subfigure}
    \hfill
    \begin{subfigure}[b]{0.4\linewidth}
    \includegraphics[width=\linewidth]{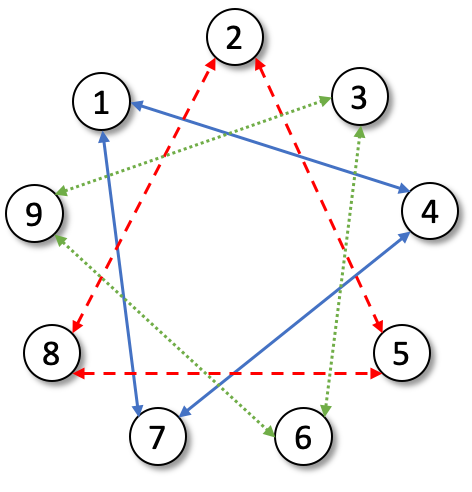}
    \end{subfigure}
    \hfill
    \caption{Communication Group Partition Example for 9 Users}
    \label{fig:group_example}
\end{figure}

Figure~\ref{fig:group_example} shows an example of our group construction: a set of partitions of the nine users $\{1,\dots,9\}$ into groups (of size $s=3$) with a gap constraint.
Each of the four partitions corresponds to a communication scheme in an ADMM iteration.
The members of a group (triangles) are free to communicate their $y$ values within themselves in an iteration.
These 4 partitions create a communication gap ($t_g = 4$) over the ADMM aggregation.
According to Theorem~\ref{lemma:max_iter}, users in Figure~\ref{fig:group_example} do not reveal privacy information if the aggregation converges in less than eight iterations ($2t_g = 8$).


\subsection{Connection to Combinatorial Design Theory}
We connect our group formation problem to a well-known problem in combinatorial design theory. We need a couple of definitions from the combinatorial block design theory.

\begin{definition}[Balanced Incomplete Block Design (BIBD)]
A balanced incomplete block design (BIBD) $B(n,s,t)$ is a pair of $(N,q)$ such that,
\begin{itemize}
    \item $N$ is a set of $n$ elements.
    \item $q$ is a collection of subsets of $N$.
    \item Each element of $q$ (called blocks) is of size exactly $s$.
    \item Every pair of distinct elements of $N$ is contained in exactly $t$ blocks of $q$. 
\end{itemize}
\end{definition}

If $t=1$, the blocks are pairwise disjoint. We define a parallel class as follows.

\begin{definition}[Parallel Classes and Resolution]
A subset  $\pi \subseteq q$ is called a partial parallel class provided that the blocks of $\pi$ are pairwise disjoint. 
If $\pi$ partitions $N$ (i.e., the intersection of the blocks are pairwise empty and the union of them covers $N$), 
then $\pi$ is a parallel class. The number of parallel classes in $q$ is called a resolution.
$(N,q)$ is said to be a resolvable BIBD if $q$ has at least one resolution. 
\end{definition}

$q$ may have more than one parallel classes, i.e., more than one partitions of $N$.
For example, if $N=\{1,2,3,4\}$ then $q=\{\{1,2\},\{3,4\},\{1,3\},\{2,4\}\}$ contains two parallel classes (partitions) of $N$: $\{\{1,2\},\{3,4\}\}$ and $\{\{1,3\},\{2,4\}\}$. The resolution of $q$ is 2. 
We refer the reader to \cite{stinson2007combinatorial} for comprehensive references of block design problems.

With these definitions, we are now ready to establish a connection between our group construction and a balanced incomplete block design problem. 

\begin{lemma}
A BIBD $B(n,s,1)$ with resolution $k$ implies a gap of $k$ with group size $s$.
\end{lemma}

\begin{proof}
Since $t=1$, the blocks are pairwise disjoint. A resolution of $k$ ensures $k$ partitions of all users with the block size of exactly $s$. These blocks are groups in our construction.
Each partition provides a communication scheme in each iteration.
\end{proof}

Naturally we want to maximize the resolution since a large gap size allows more iterations under the privacy guarantee (Theorem~\ref{lemma:max_iter}) and more iterations minimize estimation error (Lemma~\ref{lemma:convergence}). 

A special case of $B(n,s,1)$ is when $n \equiv 3 \bmod 6$ and $s=3$. This is known as a Kirkman triple system~\cite{abel1996kirkman}. 
The goal is to find $\frac{n-3}{2}+1$ parallel classes. 
Ray-Chaudhuri~et~al.~\cite{ray1971solution} propose a solution for Kirkman triple systems.
If $s=4$, the corresponding design problem is known as the Social Golfer Problem~\cite{colbourn2010crc}.
The task here is to maximize the number of parallel classes for any $n$ divisible by 4.
Unfortunately, this problem still remains unsolved~\cite{soc_golf_mathworld}.  

A constrained programming solver can be used to construct a solution.
In \cite{barnier2005solving}, the authors propose different pruning strategies to speed up the problem-solving time. But these techniques are based on exhaustive search. They scale poorly as $n$ increases. 

Recently there have been works on the randomized analysis of combinatorial design problems~ \cite{kwan2016almost,ferber2019almost,keevash2018hypergraph}.
Motivated by them, we design a randomized heuristic algorithm to solve a Kirkman triple system problem.
Although our algorithm is not guaranteed to find partitions with the optimal resolution, it does provide partitions that  guarantee the desired privacy.
\vspace{0cm}
\subsection{Randomized Group Construction Algorithm} \label{sec:randomized_grouping_algorithm}

We intend to generate a collection of partitions ($\Pi$) of $N$. Here $N$ is a vertex set of $n$ users. Each of these partitions consists of pairwise disjoint groups of size 3. Before discussing our randomized algorithm, let us first state a few preliminaries. 

Let $G$ be a complete graph, and $H$ be a 3-uniform hypergraph; both defined on the vertex set $N$. In graph theory, a \emph{matching} in a hypergraph is a collection of independent sets of edges, i.e., the matching edges are vertex disjoint. A \emph{perfect matching} is one that covers the entire vertex set. A \emph{clique} of size $k$ in a graph is a complete subgraph with $k$ vertices.  A \emph{k-clique} partition of a graph is a partition of the set of vertices into cliques of size $k$. If $k=3$, then it is called a triangle, and a corresponding partition is called a triangle partition.

We observe that a parallel class in a Kirkman triple system is a perfect matching in $H$. 
These matching edges also correspond to a triangle partition $\pi$ of $G$.

\begin{algorithm}[htb]
\caption{Randomized-KirkmanTriple}
\label{alg:rand-bibd}
\begin{algorithmic}
\State $G$: Complete graph with nodes $n$
\State $N:\{1,2,\dots,n\}$
\State $\Pi = \emptyset$, $\pi = \emptyset$
\Repeat
    \State $Q$: Randomly sampled $3$ vertices from $N$ without replacement
        \If {$Q$ forms a triangle in $G$} \Comment{A triangle is found}
        \State $\pi = \pi \cup Q$
        \State $N = N - Q$
        \State Remove the edges of $\pi$ from $G$ 
        \If{ $N = \emptyset$} \Comment{A partition is found}
            \State $\Pi = \Pi \cup \pi$
            \State $\pi = \emptyset$
            \State $N=\{1,2,\dots,n\}$ \Comment{Reset $N$}
        \EndIf
    \EndIf
\Until{Stopping criteria met}
\end{algorithmic}
\end{algorithm}

Algorithm~\ref{alg:rand-bibd} shows a randomized heuristic algorithm to generate parallel classes in a Kirkman triple system.
We start with the complete graph $G$. We randomly sample a triangle $Q$ from $G$. Then we remove $Q$ from $N$ and all the edges of $Q$ from $G$.
We keep repeating this process until we find a triangle partition (i.e., when $N$ becomes empty). 
In this case, we add the triangle partition $\pi$ into $\Pi$, reset $N$, and start sampling again. 
We continue this process until we reach the stopping criteria.

In the algorithm, we remove the edges of $Q$ from $G$ when we find a triangle $Q$. That guarantees the desired pairwise disjointness.



Note that we can easily extend the triangle removal technique to the $s$-clique removal for an arbitrary $s$. Instead of sampling a triangle, we can sample a $s$-clique. 
In Section~\ref{eval}, we will show the performance of the algorithm, i.e., the number of generated partitions 
with different $n$ for triangles and $4$-cliques.
\section{Privacy-preserving Decentralized Federated Learning Algorithm}


\begin{algorithm}
\caption{Privacy-preserving Decentralized FL Algorithm}
\label{alg:decentralized_secure_fl_algorithm}
\begin{algorithmic}
\State $N$: a set of peers, $n = |N|$ 
\State Comm. pattern $\Pi = \{\pi\}$, a set of partitions $\pi$ of $N$ into  
groups generated by Algorithm~\ref{alg:rand-bibd}

\State
\ForAll{peer $k$ \textbf{in parallel}}

    \State initialize $w_k^0$
    \State $w^0 = \mathit{SecureAggregation}(k, w_k^0)$
    \For{each FL round $r = 1, 2, ...$}
        \State $w_k^r  = Update(k, w^{r-1})$
        \State $w^r  = \mathit{SecureAggregation}(k, w_k^r)$
    \EndFor
\EndFor
\State

\Function{$\mathit{SecureAggregation}$}{$k, w_k$}
    \State initialize $\lambda^0_k$
    \State $z^0 =  \pmb{0}$
    
    \For{iteration $i = 1 $ to $I$} \Comment{$I$: max iterations}
        \State $\pi$: $(i \bmod |\Pi|)$-th partition in $\Pi$ \Comment{$|\Pi|$: gap size}
        \State $g$: group in $\pi$ such that $k \in g$
        \State $x_k^{i} = \frac{1}{2+\rho} (2w_k-\lambda_k^{i-1} + \rho z^{i-1})$
        \Comment{$x$ minimization}
        \State Send $y_k^i = x_{k}^{i} + \frac{1}{\rho} \lambda_{k}^{i-1}$ to other peers in $g$
        \State $z_g^{i} = \frac{1}{n} \sum_{u \in g} y_u^i $ \Comment{Partial $z$ sum}
        \State Send $z_g^{i}$ to the other groups ($\neq g$) 
        \State $z^{i} = \sum_{h \in \pi} z_{h}^{i}$ \Comment{Final $z$ sum}
        \State $\lambda_k^{i} = \lambda_k^{i-1} + \rho \cdot (x_k^{i}-z^{i})$ \Comment{$\lambda$ update}
    \EndFor
    \State \textbf{return} $z^I$ 
    
\EndFunction

    
    
\State
\Function{$Update$}{$k, w$}
    \For{local epoch $e = 1 $ to $E$} \Comment{$E$: \# local epochs}
        \For{mini batch $b \in \text{local dataset } D^k$}
            \State $w = w - \eta \nabla l(w;b)$ \Comment{SGD} 
        \EndFor
    \EndFor
    \State \textbf{return} $w$
\EndFunction

\end{algorithmic}
\end{algorithm}

Algorithm~\ref{alg:decentralized_secure_fl_algorithm} shows our decentralized federated learning protocol (\sysname) that preserves the privacy of each peer's local model parameters. 
Each peer initializes its local parameters and runs the secure ADMM-based aggregation ($SecureAggregation$) to agree on the initial global parameters. We will explain how the aggregation is performed later.

Similar to the classic federated learning, our algorithm consists of multiple rounds. 
In each FL round, every peer trains the model on its local dataset and updates the model parameters ($Update$). 
Next the peers synchronize the locally trained model parameters.
Instead of sending out the trained local model parameters to a central aggregator, the peers work together to compute the average via our ADMM-based secure aggregation algorithm ($SecureAggregation$).
After the aggregation finishes, all peers move to the next round.
This process repeats until the model converges.

\begin{figure}[!t]
    \centering
    \includegraphics[width=\linewidth]{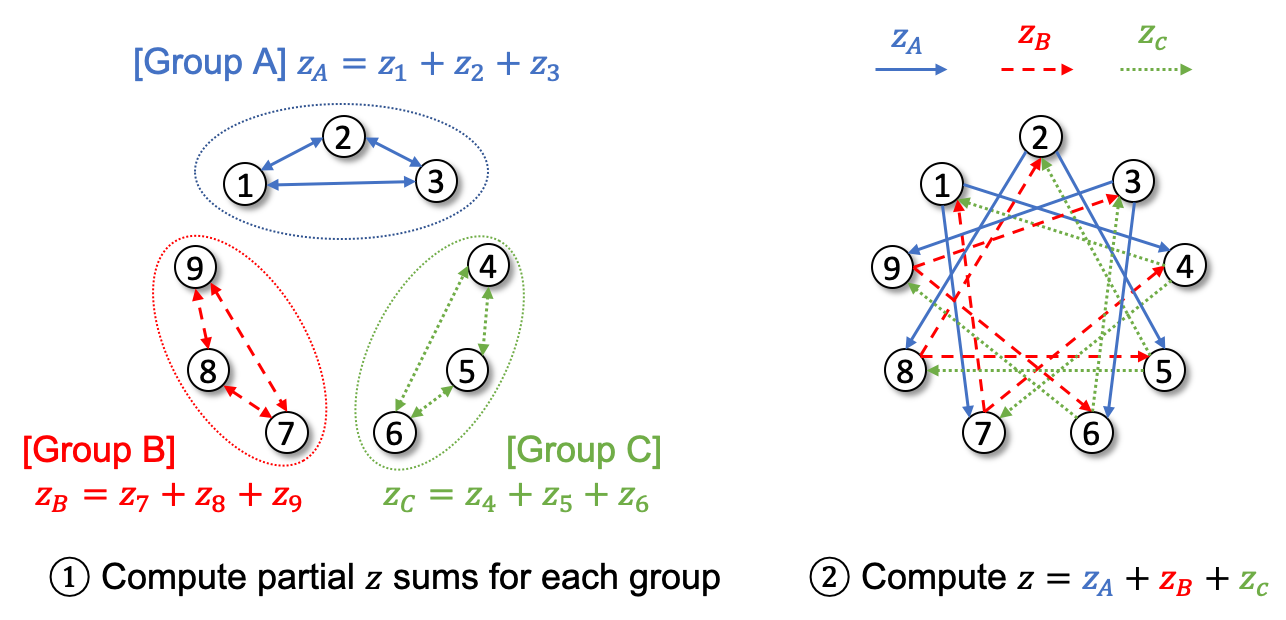}
    \caption{A Communication Example in ADMM Aggregation}
    \label{fig:admm_aggregation_example}
\end{figure}


Our secure ADMM-based aggregation consists of multiple iterations. Figure~\ref{fig:admm_aggregation_example} shows a communication example in a single ADMM iteration.
In each iteration, every peer performs the $x$ minimization and communicates its $y$ value with the other peers in the same group to compute the group's partial $z$ sum. 
Then peers in one group send its group's partial $z$ sum to peers in the other groups (inter-group communication) to compute the final $z$ sum.
Peers update their $\lambda$s with the final $z$ sum and this concludes the iteration.

Iterations are repeated until the predefined maximum iterations $I$. 
The $z^I$ would be an estimate of average of all peers' local model parameters. 



We assume that the peers know their communication pattern generated by Algorithm~\ref{alg:rand-bibd}.
This can be achieved easily. For example, the same random seed can be distributed over all peers at the beginning. 
Then each peer generates a communication pattern with the same randomized group construction and follows the same communication pattern across the entire ADMM iterations.
The communication pattern itself does not relate to any privacy-sensitive data. 
In contrast to the secret key generation in homomorphic encrpytion, the random seed generation and distribution can be done without any privacy concern.


\section{Evaluation}
\label{eval}

We implement \sysname 
on top of PyTorch \cite{pytorch} v1.5.1 that simulates the federated learning for image classification and language model tasks.
We run experiments on our local machine with an Intel i9-7960X CPU, 128 GB memory, and four NVIDIA GTX 2080 GPUs. 

In the evaluation, we answer the following questions:
(i) How does \sysname influence the quality of trained models (with respect to test accuracy)?
(ii) How does the number of ADMM iterations affect the estimation quality?
(iii) How does our randomized group construction algorithm (Algorithm~\ref{alg:rand-bibd}) generate parallel classes effectively?

\subsection{Experimental Setup}

\begin{table}[!t]
    \centering
    \caption{Dataset details}
    \begin{tabular}{c c c c}
        \toprule
        Dataset & \# Users & \# Train Samples & \# Test Samples\\
        \midrule
        FEMNIST~\cite{leafdataset} & 3,483 & 351,333 & 40,668\\
        Shakespeare~\cite{leafdataset} & 715 & 37,986 & 5,464 \\
        CIFAR-10~\cite{cifar10} & 1,000 & 50,000 & 10,000 \\
        \bottomrule
    \end{tabular}
    \label{tab:details_leaf_dataset}
\end{table}

\noindent\textbf{Datasets. }
We utilize an existing FL benchmark dataset (LEAF~\cite{leafdataset}) and generate a new federated version of CIFAR-10~\cite{cifar10} for a more complicated ML task.
We use two datasets from LEAF~\cite{leafdataset} for image classification (FEMNIST) and next character prediction (Shakespeare). 
FEMNIST is a federated version of Extended MNIST~\cite{emnist} (10 digits and 52 alphabet characters) where the dataset is partitioned based on writers. 
We use only digits data and users with at least 10 samples.
The Shakespeare dataset is built from the Complete Works of William Shakespeare \cite{shakespearecomplete} by partitioning lines in the plays based on the speaking roles.
In both datasets, we split each partition into 90\% training and 10\% test data.

To evaluate on a large ML model, we generate a federated version of CIFAR-10~\cite{cifar10}, a popular image classification benchmark dataset with 10 classes. Unlike the LEAF dataset, CIFAR-10 does not contain metadata to create non-i.i.d data.
We randomly partition training and test samples across 1,000 clients evenly.
Details about the datasets are in Table \ref{tab:details_leaf_dataset}.


\smallskip
\noindent\textbf{Models. }
For FEMNIST, we use a convolutional neural network comprising $5\times5$ and $2\times2$ convolutional layers with 32 and 64 filters each of which is followed by a $2\times2$ max-pooling layer, a 512-unit fully connected layer with ReLU activation, and a softmax layer.
For Shakespeare, we train a recurrent neural network that consists of an embedding layer with 100 dimensions, a GRU with 128 hidden units, and a softmax layer.
We use ResNet-18~\cite{resnet} for the federated version of CIFAR-10.

\smallskip
We use 9 and 15 peers representing distributed sites that train a ML model in a federated fashion.
Since the number of peers is smaller than the number of users in the datasets, we randomly assign the same number of users in the datasets to each peer. As a result, each peer uses a collection of data for those users as its local dataset.

We compare three different methods: Local-only, FedAvg~\cite{mcmahan2017communication}, and our \sysname.

In Local-Only, every peer trains a ML model over only its local dataset without any communication.

FedAvg~\cite{mcmahan2017communication} works as follows. 
A central server maintains global ML model parameters.
Peers receive the global model parameters from the server and perform training on local datasets for the pre-defined local epochs.  
Peers then send locally trained model parameters to the server.
The centralized server aggregates these model parameters by averaging them.
These steps, composing a single FL round, repeat until convergence.

For all the experiments, we train models for 50 FL rounds with a batch size of 32, a learning rate of 0.001, and the number of local epochs 1. 
We use the RMSProp~\cite{rmsprop} optimizer for FEMNIST and CIFAR-10, and the Adam~\cite{adam} optimizer for Shakespeare.

For secure aggregation, peers initialize $\lambda$ by sampling from a uniform distribution on $[0, 1)$.
To preserve privacy in the aggregation (Section \ref{sec:secure_aggregation}), we create a gap in communication by using the group construction algorithm (Algorithm~\ref{alg:rand-bibd}). With 9 and 15 peers, the algorithm generates 4 and 5 partitions (gaps), respectively.
The gap sizes of 4 and 5 preserve the privacy of the local parameters respectively up to 7 and 9 ADMM iterations (Theorem~\ref{lemma:max_iter}). 
We set the maximum ADMM iteration to 2, which not only preserves privacy ($\leq7$ and $\leq9$) but also makes the algorithm practical by minimizing the communication overhead while offering the estimation closer to true average.
We will discuss the effect of the number of ADMM iterations in detail later (Section~\ref{sec:eval_iter}).


\begin{figure*}
    \centering
    \begin{subfigure}[b]{0.32\linewidth}
        \centering
        \includegraphics[width=\textwidth]{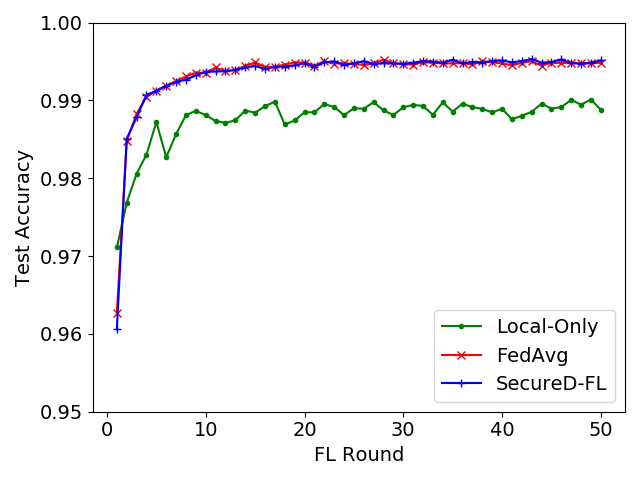}
        \caption{FEMNIST (9 peers)}
        \label{fig:lenet_val_acc_9}
    \end{subfigure}
    \begin{subfigure}[b]{0.32\linewidth}
        \centering
        \includegraphics[width=\textwidth]{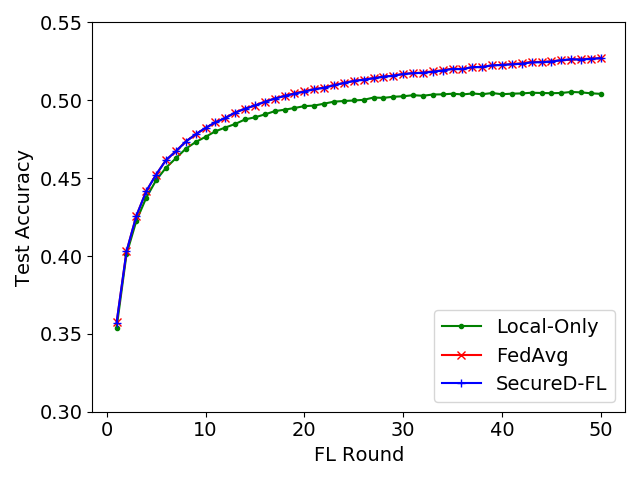}
        \caption{Shakespeare (9 peers)}
        \label{fig:rnn_val_acc_9}
    \end{subfigure}
    \begin{subfigure}[b]{0.32\linewidth}
        \centering
        \includegraphics[width=\textwidth]{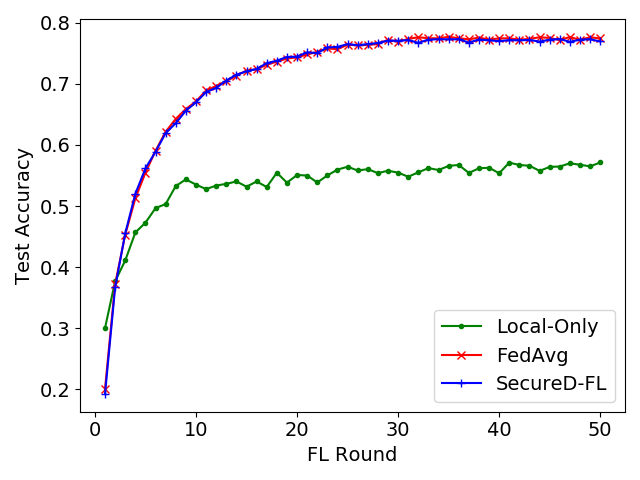}
        \caption{CIFAR-10 (9 peers)}
        \label{fig:resnet_val_acc_9}
    \end{subfigure}
    \begin{subfigure}[b]{0.32\linewidth}
        \centering
        \includegraphics[width=\textwidth]{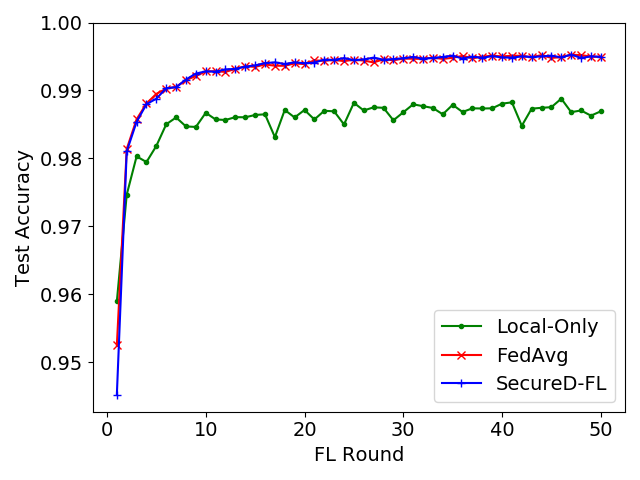}
        \caption{FEMNIST (15 peers)}
        \label{fig:lenet_val_acc_15}
    \end{subfigure}
    \begin{subfigure}[b]{0.32\linewidth}
        \centering
        \includegraphics[width=\textwidth]{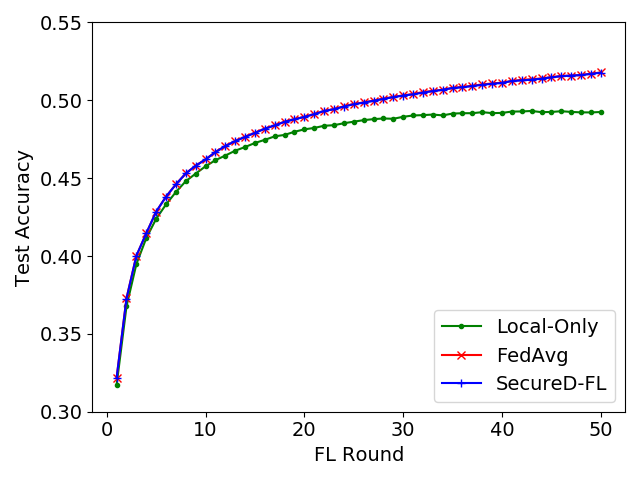}
        \caption{Shakespeare (15 peers)}
        \label{fig:rnn_val_acc_15}
    \end{subfigure}
    \begin{subfigure}[b]{0.32\linewidth}
        \centering
        \includegraphics[width=\textwidth]{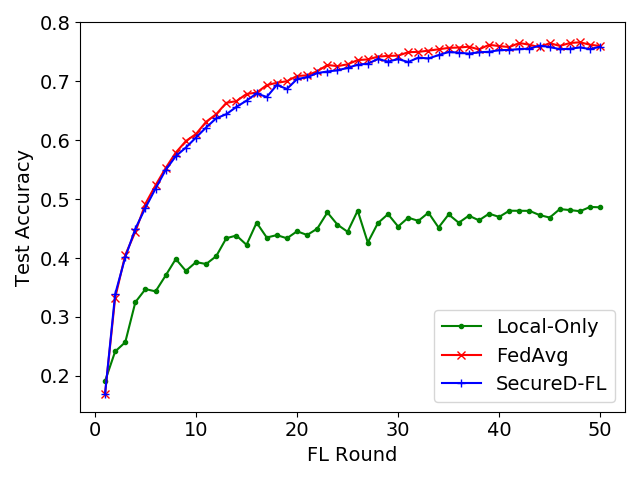}
        \caption{CIFAR-10 (15 peers)}
        \label{fig:resnet_val_acc_15}
    \end{subfigure}
    \caption{Test Accuracy per Round for Training over Different Datasets with 9 and 15 Peers. 
    }
    \label{fig:test_acc_per_epoch}
\end{figure*}

\subsection{Test Accuracy Comparison}
\label{sec:test_acc}

\begin{table}[]
    \centering
    \caption{Best Test Accuracy Over Training}
    \begin{tabular}{c c c c c}
    \toprule
    \# Peers & Datasets & Local-Only & FedAvg & SecureD-FL \\
    \midrule
    \multirow{3}{*}{9}  & FEMNIST       & 99.01\%   & 99.52\%   & 99.53\% \\
                        & Shakespeare   & 50.54\%   & 52.70\%   & 52.71\% \\
                        & CIFAR-10      & 57.13\%   & 77.69\%	& 77.33\% \\
    \cmidrule{1-5}
    \multirow{3}{*}{15} & FEMNIST & 98.88\%	& 99.52\%	& 99.53\% \\
                        & Shakespeare & 49.32\%	& 51.79\%	& 51.78\% \\
                        & CIFAR-10      & 48.68\%	& 76.64\%	& 76.08\% \\
    \bottomrule
    \end{tabular}
    \label{tab:best_test_acc}
\end{table}
Figure \ref{fig:test_acc_per_epoch} shows test accuracy per round 
over all the peers for Local-Only, FedAvg, and our \sysname with 9 and 15 peers.
Table~\ref{tab:best_test_acc} shows the best test accuracies for different configurations in the experiments.

When peers do not synchronize 
each other (Local-Only), the trained models are overfitted to the local training dataset, which results in lower test accuracy than FedAvg and \sysname. The overfitting exacerbates with more peers due to smaller local dataset.

FedAvg and \sysname show comparable test accuracies for all the datasets. For FEMNIST and Shakespeare, the test accuracy differs only less than 0.02\%.

For CIFAR-10, the test accuracy decreases 0.46\% for 9 peers and 0.73\% for 15 peers. 
ResNet-18 has more parameters than the models for the other datasets.
More parameters increase the accumulated estimation error, which adversely affects the trained model performance.
Increasing the number of peers makes the estimation error more significant.
Despite these, we do not observe any substantial degradation of test accuracy between FedAvg and \sysname. 
In all cases, the degradation is only less than 1\%.

\subsection{Trade-off: \# ADMM Iterations vs. Estimation Quality}\label{sec:eval_iter}
In this section, we explore a trade-off between the number of ADMM iterations and the quality of the estimated results. 
Although a large number of iterations in ADMM makes the estimated average closer to the actual average, this can increase the communication overhead and reveal privacy (Section~\ref{sec:secure_aggregation}). 
We explore a sweet spot that minimizes the communication overhead, keeps the estimation error low, and still provides the privacy guarantee for local model parameters.

We compare estimated values by the ADMM-based aggregation and actual averages of model parameters in real ML training. For this, we generate 10 checkpoints of model parameters (in every five FL rounds) during the training over FEMNIST by FedAvg.
We run the ADMM aggregation for each checkpoint, measure mean squared error (MSE), and calculate the averaged MSE over the 10 checkpoints.

We use 9 peers for training with a gap size of 4, which allows at most seven ADMM iterations to guarantee the privacy (Theorem~\ref{lemma:max_iter}). 
We vary the number of iterations in ADMM aggregations and calculates averaged MSEs.
Figure~\ref{fig:diff_admm_mse} shows the results of the experiment. The average MSE decreases fast as the number of iterations increases. After 4 iterations, the error drops below $10^{-13}$.

\begin{figure}[!t]
    \centering
    \includegraphics[width=0.8\linewidth]{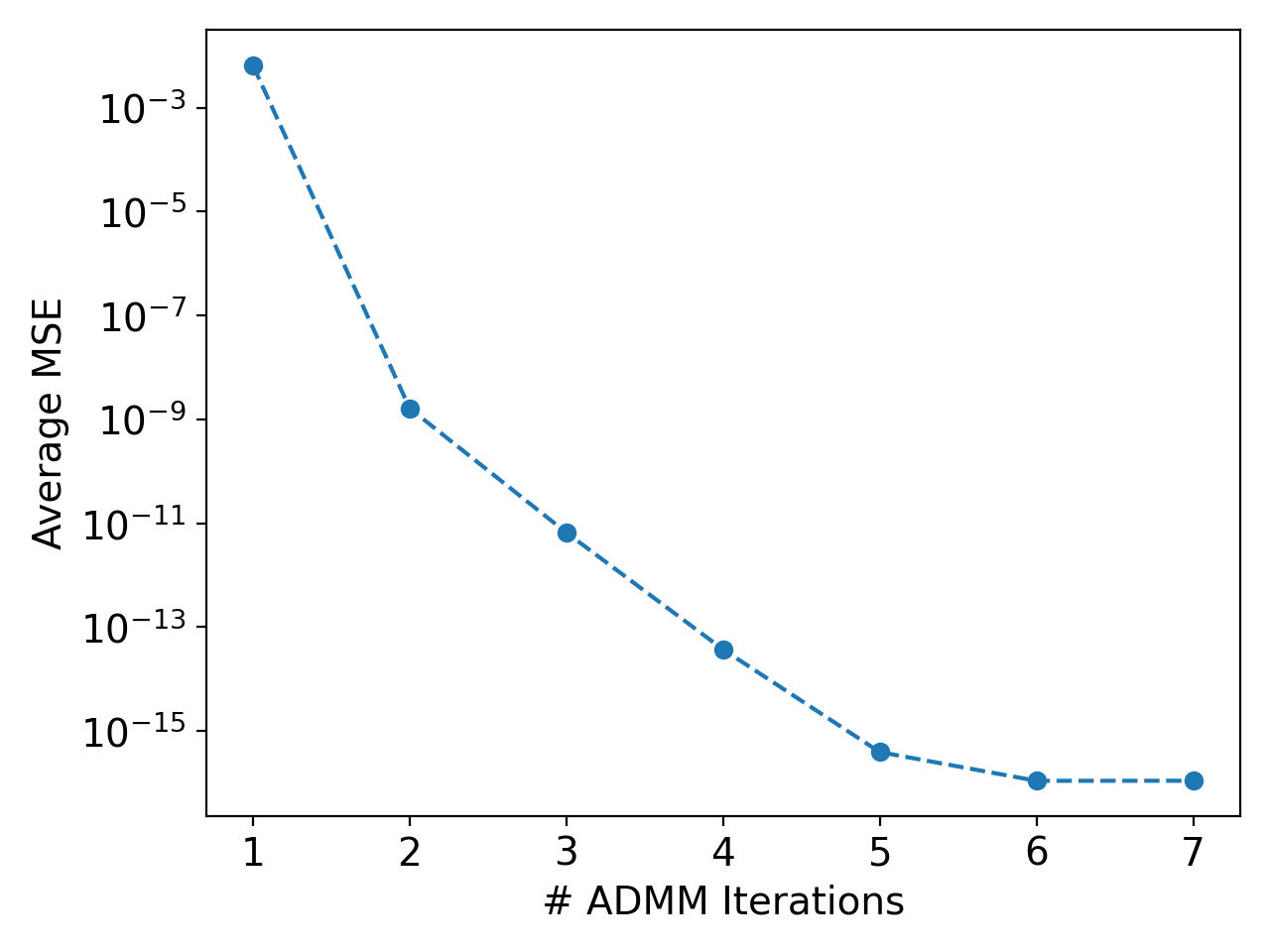}
    \caption{Average MSE with Different \# ADMM Iterations}
    \label{fig:diff_admm_mse}
\end{figure}


\begin{figure}[!t]
    \centering
    \includegraphics[width=0.8\linewidth]{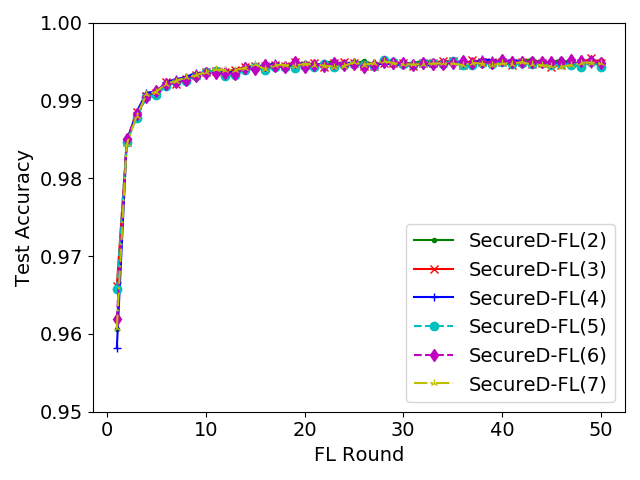}
    \caption{Test Accuracy per FL Round with Different \# ADMM Iterations (FEMNIST)}
    \label{fig:diff_admm_test_acc}
\end{figure}

To see how estimation error in ADMM aggregation affect test accuracy, we run another set of experiments. We perform training for the FEMNIST dataset with 9 peers with various numbers of ADMM iterations.
Figure~\ref{fig:diff_admm_test_acc} shows per-round test accuracy with different numbers of ADMM iterations.
For ADMM iterations greater than 1, the estimation error due to early termination does not adversely affect the test accuracy. The difference in accuracy across the ADMM iterations is negligible. So for the experiments in Section~\ref{sec:test_acc}, we set the number of ADMM iterations to 2. This decreases the communication overhead and also preserves the desired privacy.


\subsection{Randomized Group Construction Algorithm Performance}
We evaluate the randomized group construction algorithm (Algorithm \ref{alg:rand-bibd}) to construct the communication groups. 
Figure \ref{fig:perfrand} shows a plot of the our group construction algorithm performance.
We plot the number of generated parallel classes (i.e., the number of partitions of the peers) by varying the number of peers.
The red and green lines are the number of parallel classes for group sizes 3 and 4, respectively. 
We see that the number of generated partitions increases as the number of peers grows. 
This also shows the applicability of Algorithm \ref{alg:rand-bibd} to construct groups of size more than 3.

\begin{figure}
    \centering
    \includegraphics[width=\columnwidth]{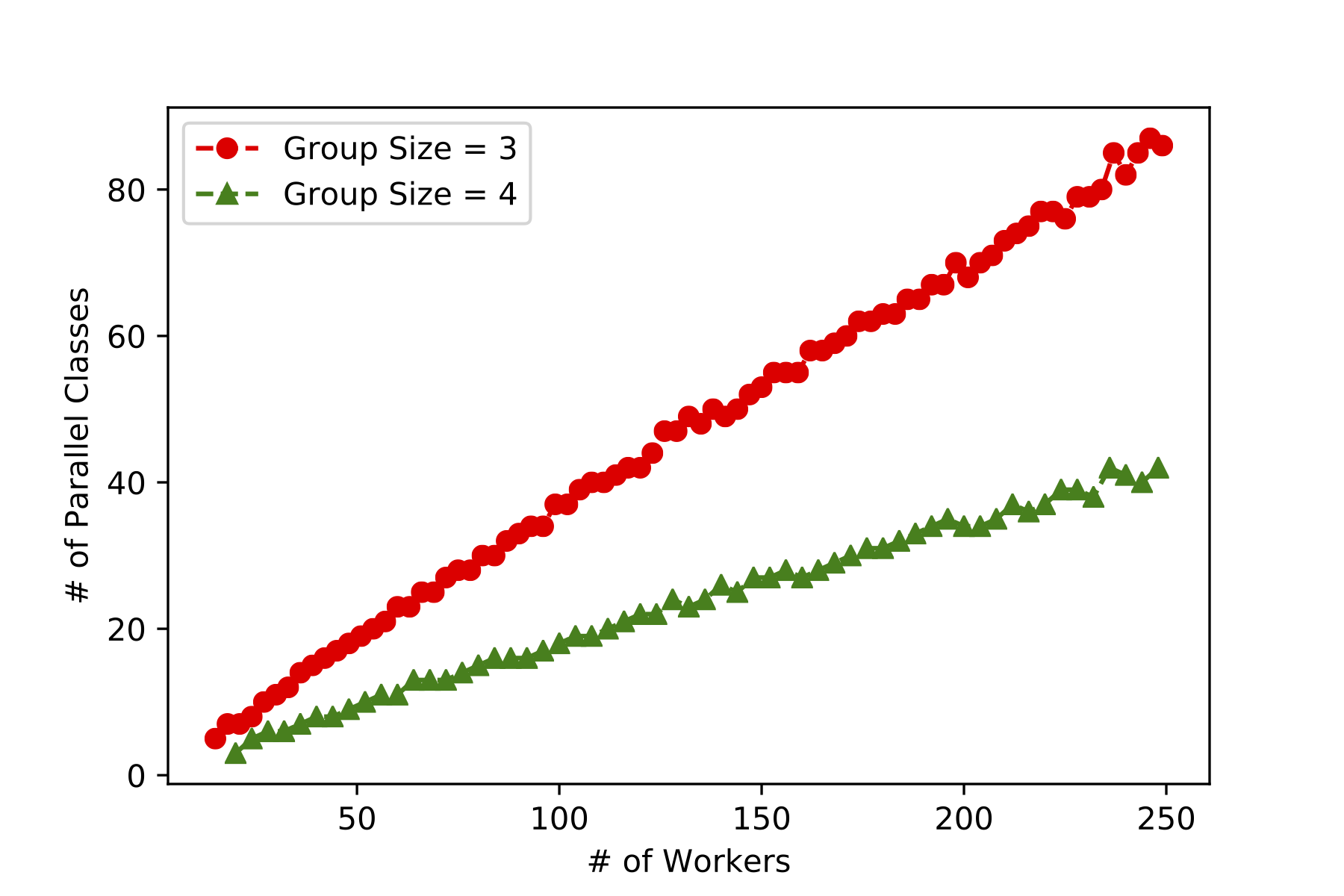}
    \caption{Group Construction Algorithm Performance
    }
    \label{fig:perfrand}
\end{figure}

\section{Related Work}
{Several approaches have been proposed to provide privacy in ADMM.} In~\cite{WEERADDANA20179502}, the authors introduce several transformation methods to alter the objective function or constraints in a way that transmitted messages are safe under an eavesdropping attack. But as we have shown in Lemma~\ref{lemma:privacy_1}, transformation-based methods are not enough for protection against honest-but-curious attacks. Differential~Privacy~(DP) 
is also considered to preserve privacy in ADMM \cite{Zhang2017dynamic, Huang2020dpadmm, zhang2018improving}. Although DP offers quantitative privacy guarantees, it is challenging to accurately estimate sensitivity and control the privacy-accuracy trade-off. {The authors in~\cite{zhang2018admm} augment ADMM with expensive homomorphic encryption which we avoid.}


In FL, Secure Multi-party Computation (SMC) techniques have been employed to preserve the privacy of transferring model parameters. PySyft~\cite{pysyft} and Chen~et al.~\cite{chen2019secure} provide the privacy guarantee by using SPDZ~\cite{spdz}. Truex, et al.~\cite{truex2018hybrid} utilizes a threshold-based addictive homomorphic encryption scheme~\cite{paillier2019homomorphic}. 
However, these approaches assume that secret keys are generated and distributed to participants securely~\cite{pysyft, chen2019secure, truex2018hybrid} or a trusted third party is essential over the entire computation~\cite{hybridalpha}. These are not required by our secure aggregation method.
Bonawitz~et al.~\cite{bonawitz2017practical} addresses the trusted third party requirement by including the secret key generation into the aggregation protocol. All these approaches to federated learning need a centralized aggregation entity.



Along with SMC, employing DP on FL has been also proposed in prior work~\cite{pysyft, truex2018hybrid, hybridalpha}, but the challenges in estimating sensitivity and privacy and model accuracy trade-off have not fully addressed. Our proposed secure aggregation algorithm is compatible with DP, which we leave as future exploration.





The authors in~\cite{zhang2019admm, ruan2017secure} propose mechanisms for privacy preserving distributed optimization using a combination of ADMM optimization and homomorphic encryption. In comparison, our approach avoids costly homomorphic encryption. DP-ADMM~\cite{Huang2020dpadmm} combines ADMM optimization with differential privacy. The authors in~\cite{zhang2018improving} propose a perturbation method for ADMM where the perturbed term is correlated with ADMM penalty parameters.


\section{Conclusion}
In this paper, we have developed a new decentralized federated learning algorithm (\sysname), focusing on the privacy preservation of ML models in training.
We identified privacy weakness in the classical ADMM-based aggregation.
To address this issue, we introduced a communication pattern (gap) that enables privacy protection of ML model parameters in the ADMM-based aggregation.
We proposed a randomized algorithm to efficiently generate this communication pattern by connecting it to \emph{balanced block design} from combinatorial mathematics.
We proved the privacy guarantee of our aggregation protocol in the honest-but-curious threat model.

Our evaluation results with the benchmark datasets showed that \sysname performed ML training comparable to the centralized standard FL method
($<0.73$\% degradation in test accuracy)
with the privacy guarantee.




\bibliographystyle{IEEEtran}
\bibliography{paper}

\end{document}